\documentclass[twoside]{article}
\usepackage[accepted]{aistats2020}

\usepackage{url}
\usepackage{dsfont}
\usepackage{amsthm}
\usepackage{amssymb}
\usepackage{amsfonts}
\usepackage{longtable}
\usepackage{caption}
\usepackage{booktabs}
\usepackage{rotating}
\usepackage{xspace}         
\usepackage{caption}
\usepackage{subcaption}
\usepackage{enumitem}

\newtheorem{proposition}{Proposition}

\newcommand{\bope}{\mbox{B-OPE}\xspace}

\usepackage[round]{natbib}

\usepackage{tikz}	
\usepackage{amsmath}	
\usepackage{amsthm}	
\newtheorem{assumption}{A}
\newtheorem{definition}{Definition}

\begin{document}

\twocolumn[

\aistatstitle{Balanced Off-Policy Evaluation in General Action Spaces}

\aistatsauthor{Arjun Sondhi \And David Arbour \And Drew Dimmery}

\aistatsaddress{Flatiron Health \And Adobe Research \And Facebook Core Data Science} 

]

\begin{abstract}
    Estimation of importance sampling weights for off-policy evaluation of contextual bandits often results in \textit{imbalance}---a mismatch between the desired and the actual distribution of state-action pairs after weighting.
    In this work we present balanced off-policy evaluation (\bope), a generic method for estimating weights which minimize this imbalance. 
    Estimation of these weights reduces to a binary classification problem regardless of action type.
    We show that minimizing the risk of the classifier implies minimization of imbalance to the desired counterfactual distribution.
    In turn, this is tied to the error of the off-policy estimate, allowing for easy tuning of hyperparameters.
    We provide experimental evidence that \bope improves weighting-based approaches for offline policy evaluation in both discrete and continuous action spaces.
\end{abstract}
Contextual bandits provide an elegant mechanism for choosing actions to optimize a given reward in the presence of uncertainty.
This is done through implementing a \textit{policy}, which defines actions performed given an observed state \citep{langford2007epoch}. 
Applications of contextual bandits abound in medicine, where personalized treatments are designed based on known patient history~\citep{tewari2017ads}, and internet marketing, where advertisements can be tailored to user interests~\citep{li2010contextual}. 
Unfortunately, in many applied settings, learning an optimal policy may be prohibitively expensive, and experimenting with an untested policy could result in unacceptably negative results, such as patient death.
Given this difficulty, an important problem area is \textit{counterfactual} or \textit{off-policy} evaluation (OPE), where the expected reward of a proposed policy is estimated using logged historical data~(states, actions, and rewards).
This problem is even more important when attempting to safely deploy a policy for an application that previously used ad-hoc or difficult-to-enumerate rules as a de facto policy.

Our work focuses specifically on improving the estimation of weights commonly used in OPE.
Because regression models often give biased results in off-policy settings, modern methods typically incorporate importance sampling to reweight the observed reward data through inverse propensity score (IPS) weighting~\citep{dudik2014doubly, thomas2016data, wang2017optimal}.
These methods have shown strong performance, but they typically assume that the importance weights (and therefore, propensity scores) are known exactly.
In practice, this is typically not the case for two reasons: (i) Many policies are high-dimensional or continuous, making it much easier to sample from the policy given a state than to know the propensity score for a given state-action pair, and (ii) Logged data may not be from a probabilistic policy, but from ad-hoc rules set by engineers, perhaps with some randomization from an A/B test.
In the absence of an oracle propensity score estimator, importance sampling weights applied to the observed state-action pairs will not necessarily result in the desired distribution.
\textit{Balance} measures the quality of the approximation to the counterfactual distribution; when weighted state-actions are imbalanced, off-policy evaluation can be arbitrarily biased.

In this paper, we develop a new estimator of balancing weights for off-policy evaluation in contextual bandit problems with arbitrary action spaces.
Our proposed estimator, which we call \textit{balanced off-policy evaluation} (\bope), is motivated by optimizing the \textit{balance} between the two policies.
\bope trains a probabilistic classifier on state-action data from both policies, which is then used to directly estimate density ratios.
These can then be plugged into any existing OPE method that involves importance sampling \citep{kallus2018policy, wang2017optimal, dudik2014doubly, farajtabar2018more}.
\bope only requires logged data on states, and the actions which would be taken by both the observed and target policies at those states.
It does not require true knowledge or an estimator of either policy density.
\bope is defined generally with respect to proper scoring rules which admit a wide variety of popular probabilistic classifiers, each of which correspond to a different underlying balance condition.

The main contributions provided in this work are:
\begin{enumerate}
    \item We introduce \bope, a balancing weight estimator for off-policy evaluation in contextual bandits that applies to arbitrary action spaces without knowledge of either the observed or proposed policy density.
    \item We show theoretically that \bope optimizes the balance condition, i.e. it minimizes a divergence between the observed and proposed state-action distributions.
    \item We show that the loss of the classification problem bounds the bias and variance of the off-policy estimate, which allows practitioners to discriminate amongst losses and perform hyperparameter tuning by using cross-validation.
\end{enumerate}

The rest of the paper is structured as follows. 
We provide an overview of OPE and the concept of balancing weights in Section~\ref{sec:background}.
We then describe \bope in Section~\ref{sec:methods}, and explain how classifier probabilities can be used to directly obtain importance sampling weights. 
In Section~\ref{sec:theory}, we provide a theoretical analysis of our estimator and prove consistency for the counterfactual policy value.
We summarize and discuss related work in Section~\ref{sec:related_work}.
In Section~\ref{sec:exps}, we evaluate \bope in numerical experiments, considering both discrete and continuous action spaces.
The latter experiments provide an extension of the ``classifier trick'' of \citet{dudik2011doubly} to continuous action spaces.

\section{Background and Problem Description}
\label{sec:background}
We assume the standard contextual bandit setup.
Our data consists of $n$ independent observations $(s_i,a_i,r_i)$.  
For each unit, a \textit{state} $s_i$ is observed, an \textit{action} $a_i$ is taken in accordance with some \textit{policy} $\pi$ (the distribution of $a\mid s$), and a \textit{reward} $r_i$ is observed in response. 
With slight abuse of notation, we use the notation $\pi$ to refer to both a policy and its density, and use $\pi(s)$ to denote the action that would be taken under policy $\pi$ for a state $s$. 

The task addressed in this work is as follows: given a proposed policy $\pi_1$ and logged data $(s, a, r)$ collected following a policy $\pi_0$ (the ``factual'' data), estimate the expected reward of following $\pi_1$ on the observed states (the ``counterfactual'' data). 
We denote the reward function as $r(a,s)$, and an estimated reward function as $\hat{r}(a,s)$.

We assume the following throughout: 
\begin{assumption} \label{assumption:support} $\pi_0(a,s) = 0 \implies \pi_1(a,s) = 0$ $\forall a \in A, s \in S$\end{assumption}
\begin{assumption} \label{assumption:boundedratio}
$0 < \frac{\pi_1(a,s)}{\pi_0(a,s)} \le \alpha_1 < \infty$
\end{assumption}
\begin{assumption}
\label{assumption:boundedreward}
$0 \leq r(a,s) \le \alpha_2 < \infty$, $\forall s, a \in S \times A$
\end{assumption}
\begin{assumption}
\label{assumption:cia}
The distribution of rewards across potential actions is independent of policy, conditional on state. 
\end{assumption}

\subsection{Off-Policy Estimation}

We now briefly review the different classes of off-policy estimation. 
Throughout this section we assume that $(s_i, a_i, r_i)$ are data collected under the observed policy $\pi_0$, and $a_i'$ is an action that would be taken under the proposed policy $\pi_1$.

The \textit{direct method} approach to this problem fits a regression model $\hat{r}(a,s)$ to approximate the reward function $r(a,s)$ under the observed policy $\pi_0$. 
The counterfactual policy value, $V_{\pi_1} := \mathbb{E}_{\pi_1}[r]$, is estimated as an average over the predicted value for actions from the new policy: $$\hat{V}^{DM} = \frac{1}{n}\sum_{i=1}^n \hat{r}(s_i, a_i')$$
In order for the resulting estimate to be consistent, the reward model $\hat{r}$ needs to generalize well to the reward distribution that would be observed under policy $\pi_1$.
In practice, this method can be badly biased if the observed state-action data is not close to the counterfactual distribution \citep{dudik2011doubly}. 

\textit{Importance sampling} reweights the observed rewards by an inverse propensity score (IPS), and a rejection sampling term, i.e.,
    $$\hat{V}^{IPS} = \frac{1}{n} \sum_{i=1}^{n} r_i \frac{\mathds{1}_{a_i}(a_i')}{\hat{\pi}_0(a_i | s_i)}
    $$
Importance sampling is unbiased when the IPS is estimated well, but it often suffers from high variance.
The \textit{self-normalized importance sampling estimator} (also called the ``weighted" or H\'{a}jek estimator) has been used to reduce variance, at the cost of small finite-sample bias, while maintaining consistency \citep{swaminathan2015self, cochran1977sampling},
$$
    \hat{V}^{SNIS} = \frac{\sum_{i=1}^{n} r_i \frac{\mathds{1}_{a_i}(a_i')}{\hat{\pi}_0(a_i | s_i)}}{\sum_{i=1}^n \frac{\mathds{1}_{a_i}(a_i')}{\hat{\pi}_0(a_i | s_i)}}
$$
In continuous action spaces, \citet{kallus2018policy} recently proposed an IPS-based method that replaces the indicator function $\mathds{1}_{a_i}(\cdot)$ with a kernel smoothing term $K$ having bandwidth $h$, i.e.,
$$
    V^{KIS} = \frac{1}{nh}\sum_{i=1}^n K\left(\frac{a_i' - a_i}{h} \right)\frac{r_i}{\hat{\pi}_0(a_i | s_i)}
$$
The corresponding self-normalized importance sampling estimator is defined analogously. 

Finally, \textit{doubly robust} estimators combine the direct method and importance sampling approaches.
These methods weigh the residuals from the direct method regression with IPS.
This reduces the variance of the resulting estimator and maintains consistency if either the direct method regression model or the importance sampling weights are correctly specified \citep{dudik2014doubly, thomas2016data}.
For discrete or continuous action spaces, the reward is estimated as
$$\hat{V}^{DR} =\frac{1}{n}\sum_{i=1}^n  \left(r_i - \hat{r}(s_i, a_i)\right)\frac{J(a_i, a'_i)}{\hat{\pi}_0(a_i | s_i)} + \hat{r}(s_i, a_i')
$$
where $J(a_i,a'_i)$ is a suitable rejection sampling term.
In a similar manner, the SWITCH estimator of \citet{wang2017optimal} combines these two estimators by using IPS unless the weight is too large, in which case it uses the direct method.

\subsection{Balance}
Estimated IPS weights for OPE do not typically ensure balanced counterfactual state-action pairs.
The true importance weights, $\rho$, imply the \textit{balance condition}
\begin{equation}
\label{eq:opebalance}
\mathds{E}_{\pi_0}\left[{\phi(a)\otimes\psi(s)}{\rho(a, s)}\right] = \mathbb{E}_{\pi_1}\left[\phi(a)\otimes\psi(s)\right],
\end{equation}
where $\phi$ and $\psi$ are any real-valued (possibly vector) functions of $a$ and $s$, respectively.
In other words, the observed state-action distribution is weighted to exactly match the target state-action distribution, resulting in consistent off-policy estimates.
However, proper propensity score model specification is difficult to test, and can be particularly difficult to obtain with continuous actions. 

Balancing weights, like \bope, seek to address this difficulty by explicitly optimizing for balance~\citep{liu2018representation, kallus2018balanced}.
If only weights which attain balance are acceptable, then directly optimizing the balance criteria is desirable.
Such estimators have been shown to provide strong results in their respective applications even under misspecification.
However, there are three main limitations to these approaches: 
(i) they focus on discrete action spaces,
(ii) they involve hyperparameters that must be set by heuristics, or
(iii) they are computationally intractable.
Our proposed estimator, \bope, is computationally simple and applies to arbitrary action spaces.
We further show that \bope minimizes the following measure of imbalance between the two policies:
\begin{definition}
\label{defn:L1dist}
Let $\phi$ and $\psi$ be real-valued functions of $a$ and $s$, respectively.
The $L_1$ functional discrepancy between the observed policy $\pi_0$ and the proposed policy $\pi_1$, with importance weights $\hat{\rho}$ is given by $\left\|\mathbb{E}_{\pi_0}\left[{\phi(a)\otimes\psi(s)}{\hat{\rho}(a, s)}\right] - \mathbb{E}_{\pi_1}\left[\phi(a)\otimes\psi(s)\right]\right\|_1$
\end{definition}
Note that a consistent estimator $\hat{\rho}$ will result in this distance going to zero by definition.

\section{Balanced Importance Sampling}
\label{sec:methods}

\textit{Balanced off-policy evaluation}~(\bope) is a simple method for estimating balancing importance sampling weights, and the central contribution of this work.
\bope leverages classifier-based density ratio estimation~\citep{sugiyama2012density, menon2016linking} to learn importance sampling ratios.
Specifically, off-policy evaluation using \bope consists of four steps: 
\begin{enumerate}[leftmargin=*, itemsep=0.5pt,topsep=1pt]
    \item Create a supervised learning problem using the concatenated proposed policy instances~$(s, a')$ and observed policy instances~$(s, a)$, as covariates and giving a label~($C$) of 0 to the observed policy and 1 to the proposed policy. 
    \item Learn a classifier to distinguish between the observed and proposed policy. 
    \item Take the importance sampling ratio as $\hat{\rho}(a_i, s_i) = \frac{\hat{p}(C = 1 | a_i, s_i)}{\hat{p}(C = 0 | a_i, s_i)}$.
    \item Replace IPS weights with the \bope estimates in any OPE method which uses them.
\end{enumerate}
Step three arrives at the importance sampler through an application of Bayes rule~\citep{bickel2009discriminative},
\begin{align*}
\frac{P(C = 1 |a,s)}{P(C=0|a,s)} 
= \frac{\pi(a,s|C=1)P(C=1)}{\pi(a,s|C=0)P(C=0)} 
= \frac{\pi_1(a,s)}{\pi_0(a,s)}
\end{align*}
where $\frac{P(C =1)}{P(C=0)} = 1$ by design.
As an example of step four, replacing IPS in the self-normalized weighting estimator would provide: 
\begin{equation}
\hat{V}^{B-OPE} = \frac{\sum_{i=1}^n J(a_i, a'_i) \hat{\rho}(a_i, s_i) r_i}{\sum_{i=1}^n J(a_i, a'_i) \hat{\rho}(a_i, s_i)}
\label{eq:bope-wis}
\end{equation}
where $J$ defines a rejection sampler term between the observed action $a_i$ and the proposed action $a'_i$. 
For discrete action spaces, this is simply $\mathds{1}_{a_i}(a_i')$.
For continuous actions, we use the kernel term of Kallus and Zhou, that is $J(a_i, a'_i) = \frac{1}{h} K(\frac{a'_i - a_i}{h})$.
We analyze the theoretical properties of this estimator in section \ref{sec:theory}, although \bope may be combined with a model of rewards.
Our experiments in Section~\ref{sec:exps} also include the performance of the SWITCH estimator of~\citet{wang2017optimal} when using \bope weights.

\bope works with a wide variety of common classification models, constrained by the following assumption:
\begin{assumption}
\label{assump:strictlyproper}
The classifier is trained using a strictly proper composite loss\footnote{A loss is strictly composite if the Bayes-optimal score is given by $\bar{s}^* = \Psi \circ \hat{p}(C=1 | s, a)$ where $\Psi$ is a link function $\Psi$ $[0, 1] \rightarrow \mathbb{R}$. Readers should see \citet{buja2005loss} and \citet{reid2010composite} for complete treatments of strictly proper composite losses.}, $\ell$, with a twice differentiable Bayes risk, $f$. 
\end{assumption}
This assumption includes a large number of widely used loss functions, such as logistic, exponential, and mean squared error, as well as models commonly used for distribution comparison, such as the kernel based density ratio estimators of \citet{sugiyama2012density}, and maximum mean discrepancy~\citep{kallus2018balanced}.

Given that \bope targets the policy density ratio, it minimizes imbalance, as given in Definition \ref{defn:L1dist}.
\begin{proposition}
\label{prop:balance}
The $L_1$ functional discrepancy between the observed policy $\pi_0$ and the proposed policy $\pi_1$ under \bope is bounded by
\begin{align*}
&\left\|\mathbb{E}_{\pi_0}\left[{\phi(a)\otimes\psi(s)}{\hat{\rho}(a, s)}\right] - \mathbb{E}_{\pi_1}\left[\phi(a)\otimes\psi(s)\right]\right\|_1 \\
&\leq \left\|\mathbb{E}_{\pi_0}\left[\phi(a)\otimes\psi({s})B(\hat{\rho}, \rho)\right]\right\|_1
\end{align*}
where $B$ is a Bregman divergence.
\end{proposition}
The proof for this proposition is in the supplement.
When $\hat{\rho} = \rho$ this discrepancy is trivially equal to 0.
The degree to which balance is attained is implied by the quality of the approximation of $\hat{\rho}$ to $\rho$.
The divergence used in this bound is determined by the classifier used in the estimation of the weights.
For example, when the \bope classifier is trained to minimize the log-loss, $B$ is the Jensen-Shannon divergence.
Bregman divergences define a wide variety of divergences including KL divergence and maximum mean discrepancy~\citep{huszar2013scoring} that are often considered in the analysis of off-policy evaluation and covariate shift~\citep{kallus2018balanced, bickel2009discriminative, gretton2009covariate}. 
Proposition 3 of \citet{menon2016linking} shows that minimizing the scoring rule in the classifier is equivalent to minimization of the divergence.
This demonstrates that minimization of the \bope classifier loss is tied to minimization of imbalance.

\section{Theoretical Properties}
\label{sec:theory}

\bope lets us formally tie classifier performance to the quality of our off-policy evaluation.
In this section, we make this explicit by:
\begin{enumerate}[itemsep=0.2pt,topsep=1pt, leftmargin=*]
    \item Describing bounds for the bias and variance of the off-policy estimate in terms of the error for the density ratio.
    \item Using results from prior work to show that minimizing the risk of the binary classifier used in \bope is equivalent to minimizing the error of the density ratio.
    \item Combining these results to show consistency of \bope~(described in Section \ref{sec:methods}) for off-policy evaluation.
\end{enumerate}
An immediate consequence of these properties is that hyper-parameter tuning and model selection to minimize the risk of the binary classifier used in \bope \textit{directly translates} to minimizing the error in off-policy evaluation via imbalance minimization.
Importantly, this property is \textit{not} shared by weights based on propensity score estimation~\citep{kang2007demystifying}. 

Let $p(a, s) := \frac{\pi_1(a,s)}{\pi_1(a,s) + \pi_0(a,s)}$ denote the true class probability of observing data $(a,s)$ under the target policy $\pi_1$ instead of the observed policy $\pi_0$. 
This is estimated with a probabilistic classifier $\hat{p}(a,s)$ on labelled state-action data. 
Additionally, let $$\rho(a,s) := \frac{\pi_1(a,s)}{\pi_0(a,s)} = \frac{p(a,s)}{1 - p(a,s)}$$ denote the true policy density ratio, with estimator $\hat{\rho}$. 
We assume the classifier has regret that decays with increasing $n$.

\begin{assumption}
\label{assump:class}
Let $\hat{p}(a,s)$ be a probabilistic classifier such that $\textrm{regret}(\hat{p}; \mathcal{D}, \ell) = O(n^{-\epsilon})$ for some constant $\epsilon \in (0,1)$.
\end{assumption}

Next, we require that our importance sampling weight estimator, $\hat{\rho}$, is independent of the observed rewards $r$. This can be easily achieved through sample splitting, training the classifier $\hat{p}$ and applying \bope on independent datasets.

\begin{assumption}
\label{assump:indep}
Given observed state-action data, the density ratio estimator $\hat{\rho}$ is independent of the observed rewards $r(\pi_0(s), s)$.
\end{assumption}

Finally, we require certain regularity conditions and rates to use in our theoretical results. 
\begin{assumption}
\label{assump:regularity}
(i) The functions $\pi_0(a,s), \pi_1(a,s), \rho(a,s)$, and  $\hat{\rho}(a,s)$ have bounded second derivatives with respect to $a$, and \\
(ii) In the continuous action domain, the bandwidth parameter $h = O(n^{-1/5})$.
\end{assumption}

We now show that the importance sampling estimator using \bope in equation \eqref{eq:bope-wis} is asymptotically unbiased, and derive a bound for its variance. 
We accomplish this by characterizing the asymptotic quantities in terms of the Bregman divergence between the estimated and true density ratios. 
In the propositions below, we use $r_{\pi_1}$ to denote $r(\pi_1(s), s)$ and $\rho_{\pi_1}$ to denote $\rho(\pi_1(s), s)$.

\begin{proposition}
\label{prop:bregmanbias}
In discrete action spaces, the expected bias of $\hat{V}^{B-OPE}$ obeys the following bound: 
$$
    \left|\mathbb{E}_{\pi_1}[r] -
    \mathbb{E}_{\pi_0}\left[ \mathds{1}_a(\pi_1(s)) \hat{\rho}(a, s) r(a, s) \right]\right|
    \leq \mathbb{E}_{\pi_0}[B\left(\rho, \hat{\rho}\right) r_{\pi_1}]
$$
In continuous action spaces, the expected bias of $\hat{V}^{B-OPE}$ obeys the following bound
\begin{align*}
    &\left|\mathbb{E}_{\pi_1}[r] -
    \mathbb{E}_{\pi_0}\left[\cfrac{1}{h} K\left(\cfrac{a - \pi_1(s)}{h} \right) \hat{\rho}(a, s) r(a, s) \right]\right|\\
    & \leq \mathbb{E}_{\pi_0}[B\left(\rho, \hat{\rho}\right) r_{\pi_1}] + o(h^2)
\end{align*}
\end{proposition}

\begin{proposition}
\label{prop:bregmanvar}
In discrete action spaces, the variance of $\hat{V}^{B-OPE}$ obeys the following bound
\begin{align*}
    &\textrm{Var}_{\pi_0}\left[ \hat{V}^{\bope} \right]\leq\\ &\cfrac{1}{n} \bigg( \mathbb{E}_{\pi_1} [ \rho_{\pi_1} r_{\pi_1}^2] + \mathbb{E}_{\pi_0} [r_{\pi_1}^2( \delta^2 + 2 \delta \rho_{\pi_1})] \bigg)
\end{align*}
In continuous action spaces, the variance of $\hat{V}^{B-OPE}$ obeys the following bound
\begin{align*}
&\textrm{Var}_{\pi_0}\left[ \hat{V}^{\bope} \right]
\leq \\
&\cfrac{ R(K)}{nh} \bigg( \mathbb{E}_{\pi_1} [ \rho_{\pi_1} r_{\pi_1}^2] + \mathbb{E}_{\pi_0} [r_{\pi_1}^2 (\delta^2 + 2 \delta \rho_{\pi_1})] \bigg) + 
o\left( \cfrac{1}{nh} \right)
\end{align*}
where $R(K) = \int K(u)^2 du$ and $\delta = B(\rho, \hat{\rho})$.
\end{proposition}

The proofs are deferred to the supplement. 
The implication of Proposition~\ref{prop:bregmanbias} is that the expected bias of \bope is bounded from above by the Bregman divergence between the true density ratio between the observed and proposed policy and the model estimate of the density ratio. 
The specific Bregman divergence depends on the choice of classifier $\hat{p}$. 
We can then appeal to Proposition 3 of \citet{menon2016linking} to provide an explicit link between the risk of the classifier and the Bregman divergence between $\rho(a, s)$ and $\hat{\rho}(a, s)$.

We now prove the consistency of the \bope estimator given in \eqref{eq:bope-wis}: 

\begin{proposition}
\label{prop:consistency}
Under Assumptions \ref{assumption:support}-\ref{assump:regularity}, and with bounded variance of the Bregman divergence, the \bope estimator is consistent for the counterfactual policy value, that is, as $n \longrightarrow \infty$, $\hat{V}^{B-OPE}  \longrightarrow \mathbb{E}_{\pi_1}[r]$.
\end{proposition}
\begin{proof}
Based on Propositions~\ref{prop:bregmanbias} and ~\ref{prop:bregmanvar}, by selecting a Bregman divergence of the form in Proposition 3 of \citet{menon2016linking}, we can bound the bias and variance in terms of the classifier $\hat{\rho}$ regret. Recall from Assumption~\ref{assump:class}, this regret scales as $O(n^{-\epsilon})$ for $\epsilon \in (0,1)$. Then, since rewards $r$ are bounded, and $h = O(n^{-1/5})$ we have that the bias tends to 0 as $n \rightarrow \infty$. 

We can apply a similar argument for the variance, by decomposing $$\mathbb{E}_{\pi_0} [B(\rho, \hat{\rho})^2] = Var_{\pi_0} [B(\rho, \hat{\rho})] + \mathbb{E}_{\pi_0} [B(\rho, \hat{\rho})]^2.$$ Then, given that $Var_{\pi_0} [B(\rho, \hat{\rho})]$, $\rho$, and $r$ are bounded, we have that the variance bound in Proposition~\ref{prop:bregmanvar} also goes to 0 as $n \rightarrow \infty$. 
\end{proof}

The full proof and technical details for these results can be found in the supplement.
It is worth briefly discussing the implications of Propositions \ref{prop:balance}-\ref{prop:bregmanvar} combined with Proposition 3 of \citet{menon2016linking} which ties classifier risk to the quality of the density ratio estimate. 
Proposition \ref{prop:balance} implies that optimizing classifier performance directly translates into optimizing the quality of the importance sampler.
In short, \bope allows for principled tradeoffs between imbalance (and the bias that comes with it) against variance in finite samples.
The bias and variance of the estimated policy evaluation can be minimized by optimizing for classifier performance. 
Because the classifier risk is directly tied to the quality of the off-policy estimate, the problem is essentially reduced to model selection for supervised learning.
As sample sizes increase, however, \bope maintains consistency and reduces imbalance.
Even under misspecification, \bope seeks to minimize imbalance.
In this case, bias will not vanish asymptotically, although imbalance will.

\section{Related Work}\label{sec:related_work}

Related work can roughly be divided into three categories: off-policy evaluation of contextual bandits, balancing estimators, and density ratio estimation. 
The most closely related work is prior work on off-policy evaluation for contextual bandits.
\citet{li2011unbiased} introduced the use of rejection sampling for offline evaluation of contextual bandit problems. 
Within the causal inference community there is a long literature on the use of doubly robust estimators~ \citep[c.f.][]{bang2005doubly, kang2007demystifying, tan2010bounded, cao2009improving}.
\citet{dudik2011doubly} later proposed the use of doubly robust estimation for off-policy evaluation of contextual bandits, combining the doubly robust estimator of causal effects with a rejection sampler.
Since then, several works have sought to minimize the variance and improve robustness of the doubly robust estimator. 
\citet{farajtabar2018more} and \citet{wang2017optimal} present work to minimize the variance of the estimators by reducing the dependence on the inverse propensity score in high variance settings. 
\citet{swaminathan2015self} use a H\'{a}jek style estimator~\citep{hajek1964asymptotic}.
Later work from \citet{thomas2015safe} and \citet{swaminathan2015self} build on this work to improve estimation. 

A second related line of work is balancing estimators. 
Under correct specification of the conditional model \citet{rosenbaum1983central} show balance of the propensity score. 
More recently, a growing literature seeks to develop balancing estimators which are robust to mis-specification.
\citet{hainmueller2012entropy} and \citet{zubizarreta2015stable} provide optimization-based procedures which define weights that are balancing but are not necessarily valid propensity scores. 
\citet{imai2014covariate} later defined an estimator which strives to find a valid propensity score subject to balancing constraints. 
This was extended to general treatment regimes by \citet{fong2018covariate}.

However, none of these directly address the problem of off-policy evaluation for contextual bandits. 
\citet{kallus2018balanced} introduces a method for balanced policy evaluation that relies on a regularized estimator that seeks to minimize the maximum mean discrepancy~\citep{gretton2012kernel}.
Calculation of weights is achieved through a quadratic program, which presents computational challenges as sample size grows large. 
It is interesting to note that the proposed evaluation procedure of \citet{kallus2018balanced} fits within the assumptions of \bope where the scoring rule is maximum mean discrepancy~(a strictly proper scoring rule) and the model is learned with variance regularization. 
The accompanying classifier can be defined via a modification of support vector machine classification~\citep{bickel2009discriminative}.
\citet{dimakopoulou2018balanced} propose balancing in the context of online learning linear contextual bandits by reweighting based on the propensity score. 
This differs from this work in the focus on online learning rather than policy evaluation and the use of a linear model-based propensity score which provides mean balance only in the case of correct specification.
\citet{wu2018variance} propose a method which seeks to minimize an $f$-divergence to minimize regret, similar to the target in this work. 
However in the setting of \citet{wu2018variance} access to the true propensities are assumed, whereas \bope estimates the density ratio directly from observed and proposed state action pairs. 

The final line of related work is density ratio estimation.
The use of classification for density ratio estimation dates back to at least \citet{qin1998inferences}.
Later work leverages classification for covariate shift adaptation~\citep{bickel2007discriminative, bickel2009discriminative} and two-sample testing~\citep{friedman2004multivariate, lopez2017revisiting}.
However, this work is the first time classifier-based density estimation has been adapted for off-policy evaluation. 
There is also a growing literature on density ratio estimation that does not rely on classification models. 
These methods largely rely on kernels to perform estimation~\citep{huang2007correcting, sugiyama2012density}.
KL importance estimation~(KLIEP)~\citep{sugiyama2008direct}, and least squares importance fitting~(LSIF)~~\citep{kanamori2009least} are the most directly relevant, given their ability to optimize hyper-parameters via cross validation.
Interestingly, \citet{menon2016linking} provides a loss for classification based density ratio estimation that reproduces KLIEP and LSIF.
Thus, these estimators can be included inside of \bope by considering the corresponding loss functions for the classifier. 

\section{Experiments}\label{sec:exps}

In the experiments that follow, we evaluate direct method, importance sampling, and SWITCH estimators for off-policy evaluation and show that estimators which use \bope typically outperform those which use standard IPS.
For the latter two methods, we compare inverse propensity score and \bope weights, and use the self-normalized versions of the estimators given in Section \ref{sec:background}.
In the SWITCH estimator, the threshold parameter $\tau$ is selected using the tuning method suggested by \citet{wang2017optimal}. 
We defer our results for doubly robust estimators to the supplement, but found the same trends in those evaluations. 
The direct method, propensity score, and \bope estimators are all trained as gradient boosted tree classifiers (or regressors for the continuous evaluations). 

\subsection{Discrete Action Spaces}

\begin{figure*}[ht]
    \centering
    \includegraphics[width=\textwidth,trim={0 1.5cm 0 0},clip]{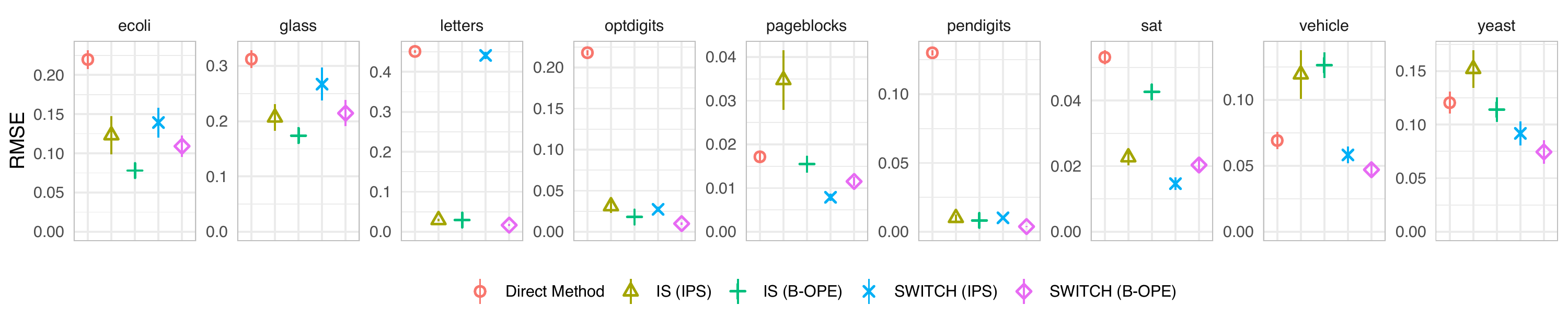}
    \includegraphics[width=\textwidth,trim={0 .5cm 0 0},clip]{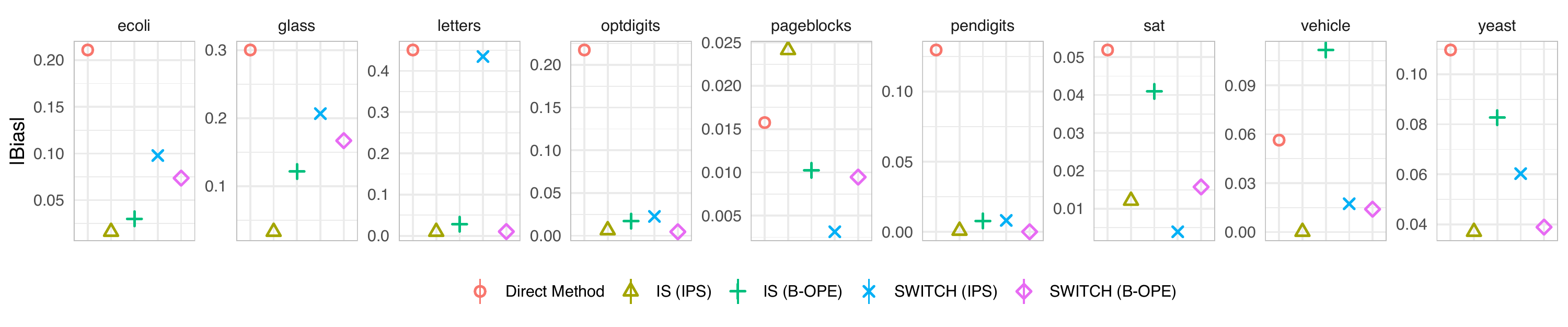}
    \caption{
    Root mean-squared error~(RMSE) and bias for discrete action spaces using the classifier trick of \citet{dudik2014doubly}.
    Points labeled ``IS'' use only weighting plus rejection sampling with either IPS or \bope.
    Points labeled ``SWITCH'' use the method of \citet{wang2017optimal} to adaptively combine importance sampling and direct estimators.
    RMSE is lower when using \bope compared to IPS; bias is typically slightly higher from \bope.
    }
    \label{fig:discresults}
\end{figure*}

We begin by evaluating the accuracy of \bope for the value of an unobserved policy in the discrete reward setting. 
We employ the method of \citet{dudik2011doubly} to turn a k-class classification problem into a k-armed contextual bandit problem. 
We split the data, training a classifier on one half of the data (\verb+train+). 
This classifier defines our target policy, wherein the action taken is the label predicted. 
The reward is defined as an indicator of whether the predicted label is the true label.
The optimal policy, then, is to take an action equal to the true label in the original data.
Evaluating a policy corresponds to estimating the actor's accuracy at identifying the true label.

In the second half of the dataset (\verb+test+) we retain only a `partially labeled' dataset where we uniformly sample actions (labels) and observe the resulting rewards. 
The \verb+train+ half of the data is also used to train direct method, propensity score, and \bope models.
OPE estimators based on these models are then applied to the \verb+test+ data to estimate the relevant quantities for off-policy evaluation methods. 
We compare the expected reward estimates to the true mean reward of the target policy applied to the \verb+test+ data. 
For each dataset, this process is repeated over 100 iterations, where we vary the actions under the observed uniform policy. 

Our target policy model is trained as a multi-class random forest classifier.
These models use the default hyperparameter values from scikit-learn with the exception of the number of trees.
In order to provide increasingly complex policies to evaluate, we increase the number of trees as a function of sample size: $\left\lfloor 10 \times n ^ {\frac{1}{4}}\right\rfloor$.
The propensity score, \bope and direct method (one-vs-rest) models are gradient boosted decision trees with default XGBoost hyperparameters with the exception of the number of boosting iterations.
In order to adapt the estimator to the size of the dataset, the number of iterations is set as a function of sample size: $\left\lceil 20 \times \sqrt{n}\right\rceil$.
We use the same datasets from the UCI repository \citep{dua2017uci} used by \citet{dudik2011doubly}, and summarize their characteristics in the supplement. 
For some datasets, we removed classes with low frequencies to avoid issues when data splitting. 

The results of the OPE estimators are summarized in Figure~\ref{fig:discresults}, where we plot the root mean squared error and bias averaged over 100 iterations. 
We see that the direct method estimator tends to be heavily biased for the true policy value, compared to \bope and IPS. 
The direct method generally performs quite poorly in terms of overall accuracy.
The standard \bope estimator performs at least as well as and typically better than the IPS estimator. 
This also holds for the corresponding SWITCH estimators.
While \bope often has slightly higher bias than IPS, it strikes a better balance between bias and variance, leading to substantially improved accuracy in most cases. 

\subsection{Continuous Action Spaces}
\label{subsection:continuous}

\begin{figure*}[ht]
    \centering
    \includegraphics[width=0.85\textwidth,trim={0 1.5cm 0 0},clip]{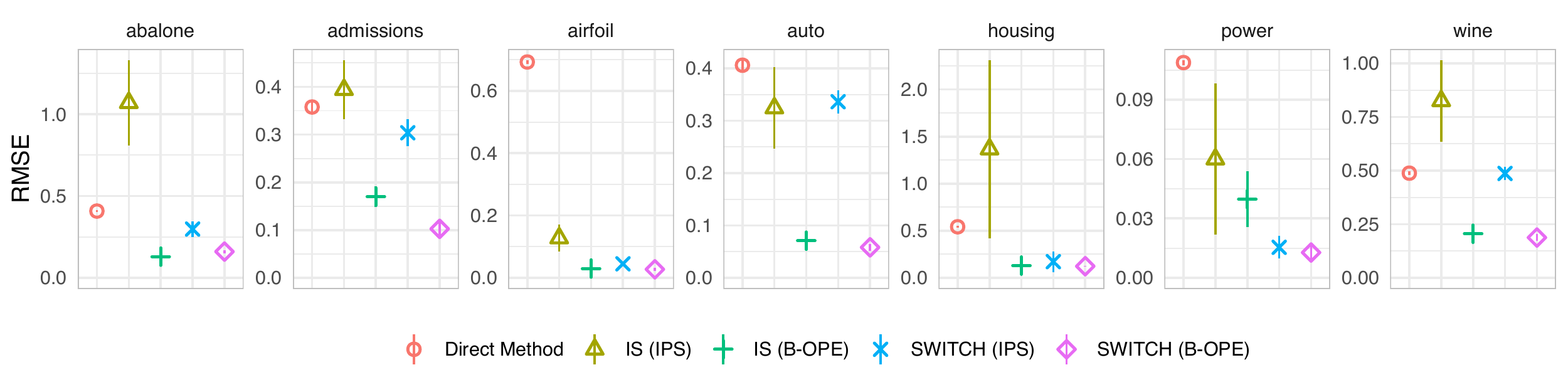}
    \includegraphics[width=0.85\textwidth,trim={0 .5cm 0 0},clip]{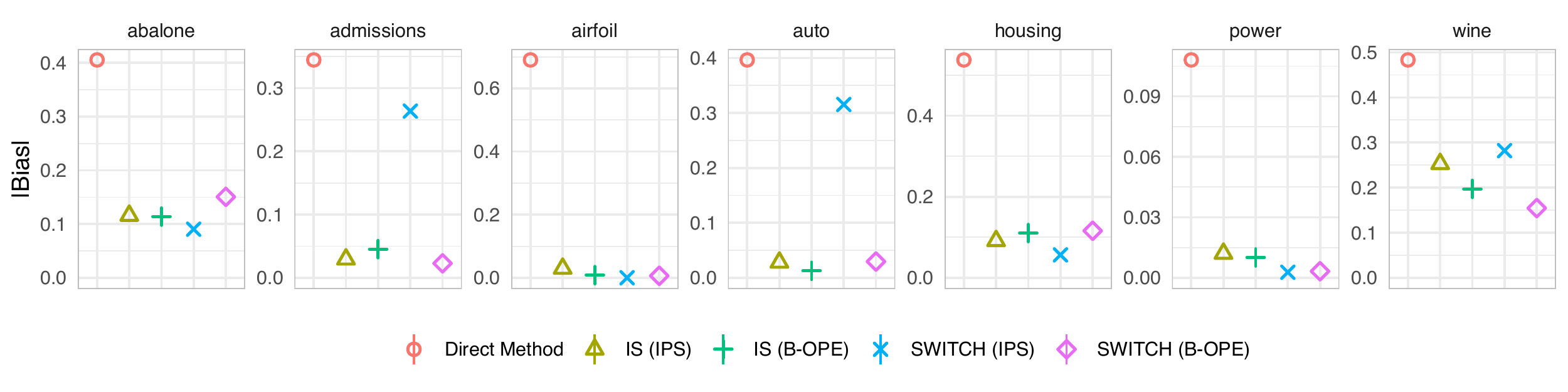}
    \caption{
    RMSE and bias for continous action spaces using a modification of the classifier trick of \citet{dudik2014doubly} for regression detailed in Section \ref{subsection:continuous}. 
    Points labeled ``IS'' use only weighting plus rejection sampling with either IPS or \bope.
    Points labeled ``SWITCH'' use the method of \citet{wang2017optimal} to adaptively combine importance sampling and direct estimators.
    RMSE and bias are typically lower using \bope than IPS.
    }
    \label{fig:ctsresults}
\end{figure*}

For the continuous action case, we provide a novel extension of the same transformation employed in the previous section for evaluation of discrete actions. 
We take a selection of datasets with continuous outcomes, and train a regression model on the \verb+train+ half of the data, which constitutes our target policy.
The reward of a prediction (defined to be an action in our evaluation) is the negative of the Euclidean distance to the true label. 
Thus, it is optimal to choose actions equal to the true outcome as in the discrete evaluation.
Evaluating the behavior policy is equivalent to estimating the mean squared error of the predictive model.

As before, we retain the \verb+test+ data for evaluation, while using the \verb+train+ data to train direct method, propensity score, and \bope models.
For our observed policy, we sample actions from the empirical distribution of \verb+train+ labels, and compute the corresponding rewards.
We then estimate the target policy value, repeating this over 300 iterations.
We retain the same basic models from the previous section for this evaluation, swapping out classifiers for regressors as appropriate.

We use datasets from the UCI repository~\citep{dua2017uci} and Kaggle, and summarize their characteristics in the supplement.
The policy we evaluate is given by training a random forest regression to predict the continuous outcome.
We also use gradient boosted regression trees for training direct method, propensity score, and \bope models. 
Specifically, to obtain a continuous propensity score, we apply our observed policy to the \verb+train+ data, and train a model $\hat{g}$ to predict actions from state features. 
Then, conditional on state $s$, the action is assumed to come from a normal distribution with mean $\hat{g}(s)$ and variance $MSE(\hat{g})$ as is standard practice~\citep{hirano2004propensity}. 
For each state-action pair $(s,a)$ in the \verb+test+ data, the generalized propensity score is then the density of this distribution at $a$.

As in the previous section, we compare \bope to IPS (with the \citet{kallus2018policy} kernel) and the direct method, including the corresponding SWITCH estimators. 
These results are displayed in Figure~\ref{fig:ctsresults}. 
We see that \bope outperforms the other methods uniformly across all datasets.
In contrast to the binary setting, \bope does a better job of correcting for bias than IPS.
This difference can be accounted for by considering that \bope estimates the densities implicitly via binary classification, while IPS must necessarily model the conditional density  of action given state.
The poor performance reflects the difficulty that many practitioners encounter when modeling continuous conditional distributions.
In addition to reducing bias, \bope greatly reduces RMSE in most datasets.
The \bope SWITCH estimator improves on the IPS version in both RMSE and bias in almost all cases.
On the \verb+power+ dataset, \bope incurs half the RMSE of IPS when used within the SWITCH estimator.
On \verb+admissions+ and \verb+auto+, \bope incurs less than one-third of the RMSE than does standard IPS.

\section{Conclusion}\label{sec:concl}

Off-policy evaluation is a critical component for the deployment of contextual bandit solutions in real-world settings. 
The efficacy of a majority of off-policy evaluation methods relies on the quality of their constituent importance weights.
As we have shown, focusing on balance provides an effective means for deriving robust importance weight estimators.
In particular, we introduced \bope, a simple, flexible, and powerful estimator of balancing weights for off-policy evaluation. 
\bope is easily implemented using off the shelf classifiers and trivially generalizes to arbitrary (e.g. continuous, multi-valued) action types. 
In Section \ref{sec:theory} we tie the bias and variance of our estimator with the risk of the classification task, and show that \bope inherently minimizes imbalance.
As a consequence of the theoretical results, hyperparameter tuning and model selection can be performed by minimizing classification error using well-known strategies from supervised learning. 
Experimental evidence indicates that \bope provides strong performance for discrete and continuous actions spaces.
A natural direction for future work is considering the case of evaluation with sequential decision making and structured action spaces. 
\bope could also be extended to perform policy optimization in all of these settings. 
It would also be interesting to integrate methods for variance reduction, e.g. \citet{thomas2016data} and \citet{farajtabar2018more}, to further improve performance.

\clearpage
\bibliographystyle{named}
\bibliography{manuscriptrefs}

\begin{thebibliography}{}

\bibitem[\protect\citeauthoryear{Acharya \bgroup \em et al.\egroup
  }{2019}]{data:admissions}
Mohan Acharya, Asfia Armaan, and Aneeta Anthony.
\newblock A comparison of regression models for prediction of graduate
  admissions.
\newblock {\em IEEE International Conference on Computational Intelligence in
  Data Science}, 2019.

\bibitem[\protect\citeauthoryear{Bang and Robins}{2005}]{bang2005doubly}
Heejung Bang and James~M Robins.
\newblock Doubly robust estimation in missing data and causal inference models.
\newblock {\em Biometrics}, 61(4):962--973, 2005.

\bibitem[\protect\citeauthoryear{Bickel \bgroup \em et al.\egroup
  }{2007}]{bickel2007discriminative}
Steffen Bickel, Michael Br{\"u}ckner, and Tobias Scheffer.
\newblock Discriminative learning for differing training and test
  distributions.
\newblock In {\em Proceedings of the 24th international conference on Machine
  learning}, pages 81--88. ACM, 2007.

\bibitem[\protect\citeauthoryear{Bickel \bgroup \em et al.\egroup
  }{2009}]{bickel2009discriminative}
Steffen Bickel, Michael Br{\"u}ckner, and Tobias Scheffer.
\newblock Discriminative learning under covariate shift.
\newblock {\em Journal of Machine Learning Research}, 10(Sep):2137--2155, 2009.

\bibitem[\protect\citeauthoryear{Buja \bgroup \em et al.\egroup
  }{2005}]{buja2005loss}
Andreas Buja, Werner Stuetzle, and Yi~Shen.
\newblock Loss functions for binary class probability estimation and
  classification: Structure and applications.
\newblock {\em Working draft, November}, 3, 2005.

\bibitem[\protect\citeauthoryear{Cao \bgroup \em et al.\egroup
  }{2009}]{cao2009improving}
Weihua Cao, Anastasios~A Tsiatis, and Marie Davidian.
\newblock Improving efficiency and robustness of the doubly robust estimator
  for a population mean with incomplete data.
\newblock {\em Biometrika}, 96(3):723--734, 2009.

\bibitem[\protect\citeauthoryear{Cochran}{1977}]{cochran1977sampling}
William~G Cochran.
\newblock {\em Sampling techniques}.
\newblock Wiley, 1977.

\bibitem[\protect\citeauthoryear{Cortez \bgroup \em et al.\egroup
  }{2009}]{data:wine}
Paulo Cortez, Ant{\'o}nio Cerdeira, Fernando Almeida, Telmo Matos, and Jos{\'e}
  Reis.
\newblock Modeling wine preferences by data mining from physicochemical
  properties.
\newblock {\em Decision Support Systems}, 47(4):547--553, 2009.

\bibitem[\protect\citeauthoryear{Dimakopoulou \bgroup \em et al.\egroup
  }{2018}]{dimakopoulou2018balanced}
Maria Dimakopoulou, Zhengyuan Zhou, Susan Athey, and Guido Imbens.
\newblock Balanced linear contextual bandits.
\newblock {\em arXiv preprint arXiv:1812.06227}, 2018.

\bibitem[\protect\citeauthoryear{Dua and Graff}{2017}]{dua2017uci}
Dheeru Dua and Casey Graff.
\newblock {UCI} machine learning repository, 2017.

\bibitem[\protect\citeauthoryear{Dud{\'\i}k \bgroup \em et al.\egroup
  }{2011}]{dudik2011doubly}
Miroslav Dud{\'\i}k, John Langford, and Lihong Li.
\newblock Doubly robust policy evaluation and learning.
\newblock In {\em Proceedings of the 28th International Conference on
  International Conference on Machine Learning}, pages 1097--1104, 2011.

\bibitem[\protect\citeauthoryear{Dud{\'\i}k \bgroup \em et al.\egroup
  }{2014}]{dudik2014doubly}
Miroslav Dud{\'\i}k, Dumitru Erhan, John Langford, Lihong Li, et~al.
\newblock Doubly robust policy evaluation and optimization.
\newblock {\em Statistical Science}, 29(4):485--511, 2014.

\bibitem[\protect\citeauthoryear{Farajtabar \bgroup \em et al.\egroup
  }{2018}]{farajtabar2018more}
Mehrdad Farajtabar, Yinlam Chow, and Mohammad Ghavamzadeh.
\newblock More robust doubly robust off-policy evaluation.
\newblock In {\em International Conference on Machine Learning}, pages
  1446--1455, 2018.

\bibitem[\protect\citeauthoryear{Fong \bgroup \em et al.\egroup
  }{2018}]{fong2018covariate}
Christian Fong, Chad Hazlett, Kosuke Imai, et~al.
\newblock Covariate balancing propensity score for a continuous treatment:
  Application to the efficacy of political advertisements.
\newblock {\em The Annals of Applied Statistics}, 12(1):156--177, 2018.

\bibitem[\protect\citeauthoryear{Friedman}{2004}]{friedman2004multivariate}
Jerome Friedman.
\newblock On multivariate goodness-of-fit and two-sample testing.
\newblock Technical report, Stanford Linear Accelerator Center, Menlo Park, CA
  (US), 2004.

\bibitem[\protect\citeauthoryear{Gretton \bgroup \em et al.\egroup
  }{2009}]{gretton2009covariate}
Arthur Gretton, Alex Smola, Jiayuan Huang, Marcel Schmittfull, Karsten
  Borgwardt, and Bernhard Sch{\"o}lkopf.
\newblock Covariate shift by kernel mean matching.
\newblock {\em Dataset shift in machine learning}, 3(4):5, 2009.

\bibitem[\protect\citeauthoryear{Gretton \bgroup \em et al.\egroup
  }{2012}]{gretton2012kernel}
Arthur Gretton, Karsten~M Borgwardt, Malte~J Rasch, Bernhard Sch{\"o}lkopf, and
  Alexander Smola.
\newblock A kernel two-sample test.
\newblock {\em Journal of Machine Learning Research}, 13(Mar):723--773, 2012.

\bibitem[\protect\citeauthoryear{Hainmueller}{2012}]{hainmueller2012entropy}
Jens Hainmueller.
\newblock Entropy balancing for causal effects: A multivariate reweighting
  method to produce balanced samples in observational studies.
\newblock {\em Political Analysis}, 20(1):25--46, 2012.

\bibitem[\protect\citeauthoryear{H{\'a}jek and
  others}{1964}]{hajek1964asymptotic}
Jaroslav H{\'a}jek et~al.
\newblock Asymptotic theory of rejective sampling with varying probabilities
  from a finite population.
\newblock {\em The Annals of Mathematical Statistics}, 35(4):1491--1523, 1964.

\bibitem[\protect\citeauthoryear{Hirano and
  Imbens}{2004}]{hirano2004propensity}
Keisuke Hirano and Guido~W Imbens.
\newblock The propensity score with continuous treatments.
\newblock {\em Applied Bayesian modeling and causal inference from
  incomplete-data perspectives}, 226164:73--84, 2004.

\bibitem[\protect\citeauthoryear{Huang \bgroup \em et al.\egroup
  }{2007}]{huang2007correcting}
Jiayuan Huang, Arthur Gretton, Karsten Borgwardt, Bernhard Sch{\"o}lkopf, and
  Alex~J Smola.
\newblock Correcting sample selection bias by unlabeled data.
\newblock In {\em Advances in neural information processing systems}, pages
  601--608, 2007.

\bibitem[\protect\citeauthoryear{Husz\'{a}r}{2013}]{huszar2013scoring}
Ferenc Husz\'{a}r.
\newblock {\em Scoring rules, divergences and information in Bayesian machine
  learning}.
\newblock PhD thesis, University of Cambridge, 2013.

\bibitem[\protect\citeauthoryear{Imai and Ratkovic}{2014}]{imai2014covariate}
Kosuke Imai and Marc Ratkovic.
\newblock Covariate balancing propensity score.
\newblock {\em Journal of the Royal Statistical Society: Series B (Statistical
  Methodology)}, 76(1):243--263, 2014.

\bibitem[\protect\citeauthoryear{Kallus and Zhou}{2018}]{kallus2018policy}
Nathan Kallus and Angela Zhou.
\newblock Policy evaluation and optimization with continuous treatments.
\newblock In {\em International Conference on Artificial Intelligence and
  Statistics}, pages 1243--1251, 2018.

\bibitem[\protect\citeauthoryear{Kallus}{2018}]{kallus2018balanced}
Nathan Kallus.
\newblock Balanced policy evaluation and learning.
\newblock In {\em Advances in Neural Information Processing Systems}, pages
  8909--8920, 2018.

\bibitem[\protect\citeauthoryear{Kanamori \bgroup \em et al.\egroup
  }{2009}]{kanamori2009least}
Takafumi Kanamori, Shohei Hido, and Masashi Sugiyama.
\newblock A least-squares approach to direct importance estimation.
\newblock {\em Journal of Machine Learning Research}, 10(Jul):1391--1445, 2009.

\bibitem[\protect\citeauthoryear{Kang \bgroup \em et al.\egroup
  }{2007}]{kang2007demystifying}
Joseph~DY Kang, Joseph~L Schafer, et~al.
\newblock Demystifying double robustness: A comparison of alternative
  strategies for estimating a population mean from incomplete data.
\newblock {\em Statistical science}, 22(4):523--539, 2007.

\bibitem[\protect\citeauthoryear{Kaya \bgroup \em et al.\egroup
  }{2012}]{data:power1}
Heysem Kaya, Pmar T{\"u}fekci, and Fikret~S G{\"u}rgen.
\newblock Local and global learning methods for predicting power of a combined
  gas \& steam turbine.
\newblock In {\em Proceedings of the International Conference on Emerging
  Trends in Computer and Electronics Engineering ICETCEE}, pages 13--18, 2012.

\bibitem[\protect\citeauthoryear{Langford and Zhang}{2007}]{langford2007epoch}
John Langford and Tong Zhang.
\newblock The epoch-greedy algorithm for contextual multi-armed bandits.
\newblock In {\em Proceedings of the 20th International Conference on Neural
  Information Processing Systems}, pages 817--824. Citeseer, 2007.

\bibitem[\protect\citeauthoryear{Li \bgroup \em et al.\egroup
  }{2010}]{li2010contextual}
Lihong Li, Wei Chu, John Langford, and Robert~E. Schapire.
\newblock A contextual-bandit approach to personalized news article
  recommendation.
\newblock In {\em Proceedings of the 19th International Conference on World
  Wide Web}, WWW '10, pages 661--670, New York, NY, USA, 2010. ACM.

\bibitem[\protect\citeauthoryear{Li \bgroup \em et al.\egroup
  }{2011}]{li2011unbiased}
Lihong Li, Wei Chu, John Langford, and Xuanhui Wang.
\newblock Unbiased offline evaluation of contextual-bandit-based news article
  recommendation algorithms.
\newblock In {\em Proceedings of the fourth ACM international conference on Web
  search and data mining}, pages 297--306. ACM, 2011.

\bibitem[\protect\citeauthoryear{Liu \bgroup \em et al.\egroup
  }{2018}]{liu2018representation}
Yao Liu, Omer Gottesman, Aniruddh Raghu, Matthieu Komorowski, Aldo~A Faisal,
  Finale Doshi-Velez, and Emma Brunskill.
\newblock Representation balancing mdps for off-policy policy evaluation.
\newblock In {\em Advances in Neural Information Processing Systems}, pages
  2649--2658, 2018.

\bibitem[\protect\citeauthoryear{Lopez-Paz and
  Oquab}{2017}]{lopez2017revisiting}
David Lopez-Paz and Maxime Oquab.
\newblock Revisiting classifier two-sample tests.
\newblock In {\em International Conference on Learning Representations}, 2017.

\bibitem[\protect\citeauthoryear{Menon and Ong}{2016}]{menon2016linking}
Aditya Menon and Cheng~Soon Ong.
\newblock Linking losses for density ratio and class-probability estimation.
\newblock In {\em International Conference on Machine Learning}, pages
  304--313, 2016.

\bibitem[\protect\citeauthoryear{Qin}{1998}]{qin1998inferences}
Jing Qin.
\newblock Inferences for case-control and semiparametric two-sample density
  ratio models.
\newblock {\em Biometrika}, 85(3):619--630, 1998.

\bibitem[\protect\citeauthoryear{Reid and Williamson}{2010}]{reid2010composite}
Mark~D Reid and Robert~C Williamson.
\newblock Composite binary losses.
\newblock {\em Journal of Machine Learning Research}, 11(Sep):2387--2422, 2010.

\bibitem[\protect\citeauthoryear{Rosenbaum and
  Rubin}{1983}]{rosenbaum1983central}
Paul~R Rosenbaum and Donald~B Rubin.
\newblock The central role of the propensity score in observational studies for
  causal effects.
\newblock {\em Biometrika}, 70(1):41--55, 1983.

\bibitem[\protect\citeauthoryear{Siebert}{1987}]{data:vehicle}
J.P. Siebert.
\newblock {\em Vehicle Recognition Using Rule Based Methods}.
\newblock TIRM--87-018. Turing Institute, 1987.

\bibitem[\protect\citeauthoryear{Sugiyama \bgroup \em et al.\egroup
  }{2008}]{sugiyama2008direct}
Masashi Sugiyama, Taiji Suzuki, Shinichi Nakajima, Hisashi Kashima, Paul von
  B{\"u}nau, and Motoaki Kawanabe.
\newblock Direct importance estimation for covariate shift adaptation.
\newblock {\em Annals of the Institute of Statistical Mathematics},
  60(4):699--746, 2008.

\bibitem[\protect\citeauthoryear{Sugiyama \bgroup \em et al.\egroup
  }{2012}]{sugiyama2012density}
Masashi Sugiyama, Taiji Suzuki, and Takafumi Kanamori.
\newblock {\em Density ratio estimation in machine learning}.
\newblock Cambridge University Press, 2012.

\bibitem[\protect\citeauthoryear{Swaminathan and
  Joachims}{2015}]{swaminathan2015self}
Adith Swaminathan and Thorsten Joachims.
\newblock The self-normalized estimator for counterfactual learning.
\newblock In {\em advances in neural information processing systems}, pages
  3231--3239, 2015.

\bibitem[\protect\citeauthoryear{Tan}{2010}]{tan2010bounded}
Zhiqiang Tan.
\newblock Bounded, efficient and doubly robust estimation with inverse
  weighting.
\newblock {\em Biometrika}, 97(3):661--682, 2010.

\bibitem[\protect\citeauthoryear{Tewari and Murphy}{2017}]{tewari2017ads}
Ambuj Tewari and Susan~A Murphy.
\newblock From ads to interventions: Contextual bandits in mobile health.
\newblock In {\em Mobile Health}, pages 495--517. Springer, 2017.

\bibitem[\protect\citeauthoryear{Thomas and Brunskill}{2016}]{thomas2016data}
Philip Thomas and Emma Brunskill.
\newblock Data-efficient off-policy policy evaluation for reinforcement
  learning.
\newblock In {\em International Conference on Machine Learning}, pages
  2139--2148, 2016.

\bibitem[\protect\citeauthoryear{Thomas}{2015}]{thomas2015safe}
Philip~S Thomas.
\newblock {\em Safe reinforcement learning}.
\newblock PhD thesis, University of Massachusetts Libraries, 2015.

\bibitem[\protect\citeauthoryear{Tüfekci}{2014}]{data:power2}
Pınar Tüfekci.
\newblock Prediction of full load electrical power output of a base load
  operated combined cycle power plant using machine learning methods.
\newblock {\em International Journal of Electrical Power and Energy Systems},
  60:126 -- 140, 2014.

\bibitem[\protect\citeauthoryear{Wang \bgroup \em et al.\egroup
  }{2017}]{wang2017optimal}
Yu-Xiang Wang, Alekh Agarwal, and Miroslav Dud{\'\i}k.
\newblock Optimal and adaptive off-policy evaluation in contextual bandits.
\newblock In {\em International Conference on Machine Learning}, pages
  3589--3597, 2017.

\bibitem[\protect\citeauthoryear{Wu and Wang}{2018}]{wu2018variance}
Hang Wu and May Wang.
\newblock Variance regularized counterfactual risk minimization via variational
  divergence minimization.
\newblock In {\em International Conference on Machine Learning}, pages
  5349--5358, 2018.

\bibitem[\protect\citeauthoryear{Zubizarreta}{2015}]{zubizarreta2015stable}
Jos{\'e}~R Zubizarreta.
\newblock Stable weights that balance covariates for estimation with incomplete
  outcome data.
\newblock {\em Journal of the American Statistical Association},
  110(511):910--922, 2015.

\end{thebibliography}

\clearpage
\onecolumn
\appendix
\section{Appendix}

\subsection{Proposition 3 of \citet{menon2016linking}}

\begin{proposition}
\label{prop:bregman}
Let $P$ be the class conditional $p(C=1 | s, a)$ and $Q$ be the class conditional $p(C = 0 | s, a)$ with marginal class probability $\frac{1}{2}$. 
Let $\mathcal{D}(P, Q, \frac{1}{2})$ be the joint distribution over $C, S, A$ decomposed into $P$ and $Q$ and the marginal $p(C) = \frac{1}{2}$.
Under assumption A\ref{assump:strictlyproper}, for any scorer $\bar{s} : \mathcal{X} \rightarrow \mathbb{R}$,
    $\textrm{regret}(\bar{s}; \mathcal{D}, \ell) = \frac{1}{2}\mathbb{E}_{X \sim Q}\left[B_{f^\circledast}(\rho, \hat{\rho})\right]$,
where $f^\circledast(z) = (1 + z) f\left(\frac{z}{1 + z}\right)$. 
\end{proposition}
The proof can be found in \citet{menon2016linking}.

\subsection{Proofs of technical results}

Here, we provide technical proofs of our results.

\subsubsection{Proof of Proposition \ref{prop:balance}}
\begin{proof}
\begin{align*}
&\left\|\mathbb{E}_{\pi_0}\left[{\phi(a)\otimes\psi(s)}{\hat{\rho}}\right] - \mathbb{E}_{\pi_1}\left[\phi(a_i)\otimes\psi(s)\right]\right\|_1\\
=&\left\|\mathbb{E}_{\pi_0}\left[{\phi(a)\otimes\psi(s)}{\hat{\rho}(a, s)}\right] - \mathbb{E}_{\pi_0}\left[\phi(a)\otimes\psi(s)\rho\right]\right\|_1\\    
=&\left\|\sum_i^N {\phi(a_i)\otimes\psi(s_i)}{\hat{\rho}(a_i, s_i)} \pi_0(a_i, s_i) - \sum_i^N \phi(s_i)\otimes\psi(s_i)\frac{\pi_1(a_i, s_i)}{\pi_0(a, s)}\pi_0(a_i, s_i)\right\|_1\\
=&\left\|\sum_i^N \phi(a_i)\psi(s_i){\pi_0(a_i, s_i)}{\hat{\rho}(a_i, s_i)}  - \phi(a_i)\otimes\psi(s_i)\pi_1(a_i, s_i)\right\|_1\\
=&\left\|\sum_i^N \phi(a_i)\psi(s_i){p(a_i, s_i)}{\left(\rho(a_i, s_i) + (\hat{\rho}(a_i, s_i) - \rho(a_i, s_i))\right)}  - \phi(a_i)\otimes\psi(s_i)\pi_1(a_i, s_i)\right\|_1\\
=&\left\|\sum_i^N  \phi(a_i)\otimes\psi(s_i)\pi_0(a_i, s_i)(\hat{\rho}(a_i, s_i) - \rho(a_i, s_i)) \right\|_1\\
=&\left\|\mathbb{E}_{\pi_0(a, s)}\left[\phi(a)\otimes\psi(x)(\hat{\rho} - \rho)\right]\right\|_1 
\leq \left\|\mathbb{E}_{\pi_0(a, s)}\left[\phi(a)\otimes\psi(x)B(\hat{\rho}, \rho)\right]\right\|_1 
\end{align*}
\end{proof}

\subsubsection{Proof of Proposition~\ref{prop:bregmanbias}}

Because the weights in the denominator $\hat{V}^{\bope}$ are each consistent for 1, we have that the sum is consistent for $n$. Therefore, by the continuous mapping theorem, we can consider the expectation of a single term in the $\hat{V}^{\bope}$ numerator.

Recall that $\rho(a, s) = \frac{\pi_1(a, s)}{\pi_0(a, s)}$ denotes the true density ratio and $\hat{\rho}(a,s)$ is the estimated density ratio. Further let $\delta(a, s) =  \hat{\rho}(a,s) - \rho(a, s)$. First, we consider the discrete action setting. We can express the expectation as:
\begin{align*}
    \mathbb{E}_{\pi_0}\left[\mathds{1}_a(\pi_1(s)) \hat{\rho}(a, s) r(a, s) \right] 
    &= \mathbb{E}_{\pi_0}\left[\mathds{1}_a(\pi_1(s)) (\rho(a, s) + \delta(a,s)) r(a, s) \right] \\
    &= \mathbb{E}_{\pi_0}\left[\mathds{1}_a(\pi_1(s)) \rho(a, s) r(a, s) \right] + \mathbb{E}_{\pi_0}\left[\mathds{1}_a(\pi_1(s)) \delta(a,s)) r(a, s) \right]
\end{align*}

We can show that the first term is equal to the policy value of $\pi_1$, while the second term provides the estimator's bias. Considering the first term, we have:
\begin{align*}
    \mathbb{E}_{\pi_0}\left[\mathds{1}_a(\pi_1(s)) \rho(a, s) r(a, s) \right]
    &= \sum_{(a,s)} \mathds{1}_a(\pi_1(s)) \rho(a, s) r(a, s) \pi_0(a,s)  \\
    &= \sum_{(a,s)} \mathds{1}_a(\pi_1(s)) r(a, s) \pi_1(a,s) \\
    &= \sum_{s} r(\pi_1(s), s) \pi_1(\pi_1(s), s) \\
    &= \mathbb{E}_{\pi_1}\left[r_{\pi_1} \right],
\end{align*}
where $r_{\pi_1}$ denotes $r(\pi_1(s), s)$.

Now, considering the bias term, and bounding $\delta$ with the Bregman divergence between $\rho$ and $\hat{\rho}$, we have:
\begin{align*}
    \mathbb{E}_{\pi_0}\left[\mathds{1}_a(\pi_1(s)) \delta(a,s)) r(a, s) \right]
    &= \sum_{(a,s)} \mathds{1}_a(\pi_1(s)) \delta(a, s) r(a, s) \pi_0(a,s)  \\
    &\le \sum_{(a,s)} \mathds{1}_a(\pi_1(s)) B(\rho, \hat{\rho}) r(a, s) \pi_0(a,s) \\
    &=  \sum_{s} B(\rho, \hat{\rho}) r(\pi_1(s), s) \pi_0(\pi_1(s),s) \\
    &= \mathbb{E}_{\pi_0}\left[B(\rho, \hat{\rho}) r_{\pi_1} \right]
\end{align*}

We now move on to the continuous action setting. We can express the expectation as:
\begin{align*}
    &\mathbb{E}_{\pi_0}\left[\cfrac{1}{h} K\left(\cfrac{a - \pi_1(s)}{h} \right) \hat{\rho}(a,s) r(a,s) \right]\\
    &=
    \int \cfrac{1}{h} K\left(\cfrac{a - \pi_1(s)}{h} \right)\left(\rho(a, s) + \delta(a, s)\right) r(a,s) \pi_0(a,s) d(a,s)\\
    &=
    \int \cfrac{1}{h} K\left(\cfrac{a - \pi_1(s)}{h} \right) \rho(a, s) r(a,s) \pi_0(a,s) d(a,s) \\
    &\hspace{10pt}+ \int \cfrac{1}{h} K\left(\cfrac{a - \pi_1(s)}{h} \right)\delta(a, s) r(a,s) \pi_0(a,s) d(a,s)
\end{align*}

We can show that the first term is equal to the true counterfactual policy value, while the second term describes the bias induced from estimating the density ratio. Considering the first term, we have:
\begin{align*}
    \int \cfrac{1}{h} K\left(\cfrac{a - \pi_1(s)}{h} \right) \cfrac{\pi_1(a,s)}{\pi_0(a,s)} r(a,s) \pi_0(s,a) d(s,a)
    &= \int \cfrac{1}{h} K\left(\cfrac{a - \pi_1(s)}{h} \right) r(a,s) \pi_1(a,s) d(s,a) 
\end{align*}

Let $u = \cfrac{a - \pi_1(s)}{h}$. Thus, $a = \pi_1(s) + hu$ and $da = h du$. Then, taking a second-order Taylor expansion of $\pi_1$ around $\pi_1(s)$:
\begin{align*}
    &\int \cfrac{1}{h} K\left(\cfrac{a - \pi_1(s)}{h} \right) \cfrac{\pi_1(a,s)}{\pi_0(a,s)} r(a,s) \pi_0(s,a) d(s,a) \\
    &= \int K\left(u \right) r(\pi_1(s) + hu,s) \pi_1(\pi_1(s) + hu,s) d(s,u) \\
    &= \int K\left(u \right) r(\pi_1(s),s) \pi_1(\pi_1(s),s) d(s,u) + \int K\left(u \right) r(\pi_1(s),s) \pi_1'(\pi_1(s),s)(hu) d(s,u) \\
    &\hspace{10pt}+  \int K\left(u \right) r(\pi_1(s),s) \pi_1''(\pi_1(s),s)\frac{(hu)^2}{2} d(s,u) +  \int K\left(u \right) o(h^2) r(\pi_1(s),s) d(s,u) \\
    &= \int K\left(u \right) du \int r(\pi_1(s),s) \pi_1(\pi_1(s),s) ds +  \int u K\left(u \right) du \int r(\pi_1(s),s) \pi_1'(\pi_1(s),s) h d(s,u) \\
    &\hspace{10pt}+  \int u^2 K\left(u \right) du \int \frac{h^2}{2} r(\pi_1(s),s) \pi_1''(\pi_1(s),s) ds +  \int K\left(u \right) du \int o(h^2) r(\pi_1(s),s) ds \\
    &= \int r(\pi_1(s),s) \pi_1(\pi_1(s),s) ds + o(h^2) \\
    &= E_{\pi_1}[r_{\pi_1}] + o(h^2).
\end{align*}
This result follows similarly to those in \citet{kallus2018policy}, by properties of kernels, bounded rewards, and since $\pi_1(a,s)$ has a bounded second derivative with respect to $a$.

Now, considering the bias term, we use the same $u-$substitution and Taylor expansion as before. We also bound $\delta$ by the Bregman divergence between $\rho$ and $\hat{\rho}$, yielding:
\begin{align*} 
    \int \cfrac{1}{h} K\left(\cfrac{a - \pi_1(s)}{h} \right)\delta(a, s) r(a,s) \pi_0(a,s) d(a,s) 
    \leq& \int \cfrac{1}{h} K\left(\cfrac{a - \pi_1(s)}{h} \right) B(\rho, \hat{\rho}) r(a,s) \pi_0(a,s) d(a,s) \\
    =& \int K\left(u \right) B(\rho, \hat{\rho}) r(\pi_1(s) + hu) \pi_0(\pi_1(s) + hu, s) d(u,s) \\
    =& \int K\left(u \right) du \int B(\rho, \hat{\rho}) r(\pi_1(s), s) \pi_0(\pi_1(s), s) ds + Rem(h) \\
    =& \int B(\rho, \hat{\rho}) r(\pi_1(s), s) \pi_0(\pi_1(s), s) ds + o(h^2) \\
    =& E_{\pi_0}[B(\rho, \hat{\rho}) r_{\pi_1}] + o(h^2)
\end{align*}

\subsubsection{Proof of Proposition~\ref{prop:bregmanvar}}

We consider the second moment of a single numerator term, and write the estimator in terms of $\rho$ and $\delta$ as above. We first consider the discrete action setting. 
\begin{align*}
    \mathbb{E}_{\pi_0}\left[\mathds{1}_a(\pi_1(s))^2 \hat{\rho}(a, s)^2 r(a, s)^2 \right] 
    &= \mathbb{E}_{\pi_0}\left[\mathds{1}_a(\pi_1(s)) (\rho(a, s) + \delta(a,s))^2 r(a, s)^2 \right] \\
    &= \mathbb{E}_{\pi_0}\left[\mathds{1}_a(\pi_1(s)) (\rho(a, s)^2 + \delta(a,s)^2 + 2\rho(a,s)\delta(a,s)) r(a, s)^2 \right] \\
    &= \sum_{(a,s)} \mathds{1}_a(\pi_1(s)) \rho(a, s)^2 r(a,s)^2 \pi_0(a,s) \\
    &+ \sum_{(a,s)} \mathds{1}_a(\pi_1(s)) \delta(a,s)^2 r(a,s)^2 \pi_0(a,s) \\ 
    &+ \sum_{(a,s)} \mathds{1}_a(\pi_1(s)) 2\rho(a,s)\delta(a,s) r(a,s)^2 \pi_0(a,s) \\
    &\le \sum_{s} \rho(\pi_1(s), s) r(\pi_1(s),s)^2 \pi_1(\pi_1(s),s) \\
    &+ \sum_{s} 2 B(\rho, \hat{\rho}) \rho(\pi_1(s),s) r(\pi_1(s),s)^2 \pi_0(\pi_1(s),s) \\
    &+ \sum_{s} B(\rho, \hat{\rho})^2 r(\pi_1(s),s)^2 \pi_0(\pi_1(s),s) \\
    &= \mathbb{E}_{\pi_1} [ \rho(\pi_1(s), s) r_{\pi_1}^2] + \mathbb{E}_{\pi_0} [B(\rho, \hat{\rho})^2 r_{\pi_1}^2] + \mathbb{E}_{\pi_0} [2 B(\rho, \hat{\rho}) \rho(\pi_1(s), s) r_{\pi_1}^2]
\end{align*}

Therefore, the variance of the estimator is bounded by:
$$
\cfrac{1}{n} \left( \mathbb{E}_{\pi_1} [ \rho(\pi_1(s), s) r_{\pi_1}^2] + \mathbb{E}_{\pi_0} [B(\rho, \hat{\rho})^2 r_{\pi_1}^2] + \mathbb{E}_{\pi_0} [2 B(\rho, \hat{\rho}) \rho(\pi_1(s), s) r_{\pi_1}^2] \right)
$$

Next, we consider the second moment of a term in the estimator in the continuous action setting:
\begin{align*}
    &\mathbb{E}_{\pi_0}\left[ \left( \cfrac{1}{h} K\left(\cfrac{a - \pi_1(s)}{h} \right) \hat{\rho}(a,s) r(a,s) \right)^2 \right]\\
    &=
    \int \cfrac{1}{h^2} K\left(\cfrac{a - \pi_1(s)}{h} \right)^2 \left(\rho(a, s) + \delta(a, s)\right)^2 r(a,s)^2 \pi_0(a,s) d(a,s)\\
\end{align*}

We substitute $u = \cfrac{a - \pi_1(s)}{h}$ as before. Then, $a = \pi_1(s) + hu$ and $da = h du$. 
\begin{align*}
    &\mathbb{E}_{\pi_0}\left[ \left( \cfrac{1}{h} K\left(\cfrac{a - \pi_1(s)}{h} \right) \hat{\rho}(a,s) r(a,s) \right)^2 \right]\\
    &=
    \int \cfrac{1}{h} K\left(u \right)^2 \left(\rho(\pi_1(s) + hu, s) + \delta(\pi_1(s) + hu, s)\right)^2 r(\pi_1(s) + hu)^2 \pi_0(\pi_1(s) + hu, s) d(s,u)
\end{align*}

Next, we apply a second-order Taylor series expansion of $\rho$, $\delta$, and $\pi_0$ around $\pi_1(s)$. 
Given that these functions have bounded second derivatives, we can bound the remainder by $o(h^{-1})$, as in~\citet{kallus2018policy}.
This yields:
\begin{align*}
    &\mathbb{E}_{\pi_0}\left[ \left( \cfrac{1}{h} K\left(\cfrac{a - \pi_1(s)}{h} \right) \hat{\rho}(a,s) r(a,s) \right)^2 \right]\\
    &= \int \cfrac{1}{h} K\left(u \right)^2 du \int \left(\rho(\pi_1(s), s) + \delta(\pi_1(s), s)\right)^2 r(\pi_1(s), s)^2 \pi_0(\pi_1(s), s) ds + o(h^{-1}) \\
    &= \cfrac{R(K)}{h} \int \left(\rho(\pi_1(s), s) + \delta(\pi_1(s), s)\right)^2 r(\pi_1(s), a)^2 \pi_0(\pi_1(s), s) ds + o(h^{-1}) \\
    &= \cfrac{R(K)}{h} \int \left( \rho(\pi_1(s), s)^2 + \delta(\pi_1(s), s)^2 + 2\rho(\pi_1(s), s)\delta(\pi_1(s), s) \right) r(\pi_1(s), a)^2 \pi_0(\pi_1(s), s) ds + o(h^{-1}) \\
    &= \cfrac{R(K)}{h} \left[ \int \rho(\pi_1(s), s)^2 r_{\pi_1}^2 \pi_0(\pi_1(s), s) ds + \int \delta(\pi_1(s), s)^2 r_{\pi_1}^2 \pi_0(\pi_1(s), s) ds + \int 2\rho(\pi_1(s), s)\delta(\pi_1(s), s) r_{\pi_1}^2 \pi_0(\pi_1(s), s) ds \right] \\&+ o(h^{-1}) \\
    &= \cfrac{R(K)}{h} \left[ \int \rho(\pi_1(s), s) r_{\pi_1}^2 \pi_1(\pi_1(s), s) ds + \int \delta(\pi_1(s), s)^2 r_{\pi_1}^2 \pi_0(\pi_1(s), s) ds + \int 2\delta(\pi_1(s), s) r_{\pi_1}^2 \pi_1(\pi_1(s), s) ds \right] \\&+ o(h^{-1})
\end{align*}
where $R(K) := \int K(u)^2 du$ is some constant. 

Then, bounding $\delta$ by the Bregman divergence $B$,
\begin{align*}
    &\mathbb{E}_{\pi_0}\left[ \left( \cfrac{1}{h} K\left(\cfrac{a - \pi_1(s)}{h} \right) \hat{\rho}(a,s) r \right)^2 \right]\\
    &\le \cfrac{R(K)}{h} \left[ \mathbb{E}_{\pi_1} [ \rho(\pi_1(s), s) r_{\pi_1}^2 ] + \mathbb{E}_{\pi_0} [B(\rho, \hat{\rho})^2 r_{\pi_1}^2] + \mathbb{E}_{\pi_1} [ 2 B(\rho, \hat{\rho}) r_{\pi_1}^2] \right] + o(h^{-1})
\end{align*}

Therefore, the variance of our estimator is bounded by:
$$
\cfrac{R(K)}{nh} \left( \mathbb{E}_{\pi_1} [ \rho(\pi_1(s), s) r_{\pi_1}^2 ] + \mathbb{E}_{\pi_0} [B(\rho, \hat{\rho})^2 r_{\pi_1}^2] + \mathbb{E}_{\pi_1} [ 2 B(\rho, \hat{\rho}) r_{\pi_1}^2] \right) + o\left( \cfrac{1}{nh} \right)
$$

\subsection{Evaluation details and full results} 
\label{app:simulations}

\begin{table*}[ht]
\caption{\label{tab:discdata} Summary of datasets used in discrete reward experiments}
\centering
\begin{tabular}{|l|l|l|l|l|l|l|l|l|l|}
\hline
Dataset          & ecoli & glass & letters & optdigits & page-blocks & pendigits & satimage & vehicle & yeast \\ \hline
Classes ($k$)      & 5     & 6     & 26     & 10        & 5           & 10        & 6        & 4       & 9     \\
Observations ($n$) & 327   & 214   & 20000  & 5620      & 5473        & 10992     & 6435     & 846     & 1479  \\
Covariates ($p$)   & 7     & 9     & 16     & 64        & 10          & 16        & 36       & 18      & 8     \\ \hline
\end{tabular}
\end{table*}

\begin{table*}[ht]
\caption{\label{tab:ctsdata} Summary of datasets used in continuous reward experiments}
\centering
\begin{tabular}{|l|l|l|l|l|l|l|l|}
\hline
Dataset          & abalone & admissions & airfoil & auto & housing & power & wine \\ \hline
Observations ($n$) & 4177    & 400        & 1503    & 392  & 10000   & 9568  & 1599 \\
Covariates ($p$)   & 10      & 7          & 5       & 7    & 14      & 4     & 11   \\ \hline
\end{tabular}
\end{table*}

Table~\ref{tab:disc_results} shows the results of the discrete treatment simulations.
Table~\ref{tab:cts_results} shows the results of the continuous treatment simulations.

\begin{sidewaystable}
\footnotesize
\begin{longtable}{lrrrrrrrrr}
\caption{\label{tab:disc_results}
\large Discrete evaluation\\ 
\small First row for each dataset is absolute bias, second is RMSE\\ 
} \\ 
\toprule
& & \multicolumn{4}{c}{IPS} & \multicolumn{4}{c}{\bope} \\ 
 \cmidrule(lr){3-6}\cmidrule(lr){7-10}
Dataset & Direct Method & IPS & Doubly Robust & SWITCH & SWITCH-DR & \bope & Doubly Robust & SWITCH & SWITCH-DR \\ 
\midrule
ecoli & $0.211$ & $0.016$ & $0.068$ & $0.098$ & $0.193$ & $0.030$ & $0.001$ & $0.074$ & $0.152$ \\ 
 & $0.220$ $\pm$ $0.006$ & $0.123$ $\pm$ $0.012$ & $0.189$ $\pm$ $0.015$ & $0.139$ $\pm$ $0.010$ & $0.203$ $\pm$ $0.006$ & $0.078$ $\pm$ $0.003$ & $0.112$ $\pm$ $0.009$ & $0.109$ $\pm$ $0.007$ & $0.166$ $\pm$ $0.006$ \\ 
glass & $0.300$ & $0.033$ & $0.092$ & $0.207$ & $0.290$ & $0.121$ & $0.128$ & $0.167$ & $0.235$ \\ 
 & $0.312$ $\pm$ $0.008$ & $0.207$ $\pm$ $0.012$ & $0.253$ $\pm$ $0.016$ & $0.267$ $\pm$ $0.015$ & $0.303$ $\pm$ $0.009$ & $0.174$ $\pm$ $0.008$ & $0.199$ $\pm$ $0.012$ & $0.214$ $\pm$ $0.012$ & $0.256$ $\pm$ $0.010$ \\ 
letters & $0.451$ & $0.010$ & $0.070$ & $0.435$ & $0.438$ & $0.028$ & $0.076$ & $0.010$ & $0.243$ \\ 
 & $0.451$ $\pm$ $0.002$ & $0.029$ $\pm$ $0.002$ & $0.097$ $\pm$ $0.006$ & $0.441$ $\pm$ $0.005$ & $0.439$ $\pm$ $0.003$ & $0.029$ $\pm$ $0.001$ & $0.081$ $\pm$ $0.003$ & $0.017$ $\pm$ $0.001$ & $0.245$ $\pm$ $0.003$ \\ 
optdigits & $0.217$ & $0.007$ & $0.047$ & $0.023$ & $0.179$ & $0.017$ & $0.051$ & $0.005$ & $0.128$ \\ 
 & $0.218$ $\pm$ $0.002$ & $0.031$ $\pm$ $0.004$ & $0.084$ $\pm$ $0.006$ & $0.027$ $\pm$ $0.002$ & $0.181$ $\pm$ $0.002$ & $0.018$ $\pm$ $0.000$ & $0.058$ $\pm$ $0.003$ & $0.010$ $\pm$ $0.001$ & $0.130$ $\pm$ $0.002$ \\ 
pageblocks & $0.016$ & $0.024$ & $0.034$ & $0.003$ & $0.005$ & $0.010$ & $0.011$ & $0.009$ & $0.003$ \\ 
 & $0.017$ $\pm$ $0.001$ & $0.035$ $\pm$ $0.004$ & $0.051$ $\pm$ $0.004$ & $0.008$ $\pm$ $0.000$ & $0.010$ $\pm$ $0.001$ & $0.015$ $\pm$ $0.001$ & $0.029$ $\pm$ $0.005$ & $0.012$ $\pm$ $0.001$ & $0.007$ $\pm$ $0.001$ \\ 
pendigits & $0.130$ & $0.001$ & $0.047$ & $0.008$ & $0.081$ & $0.008$ & $0.019$ & $0.000$ & $0.064$ \\ 
 & $0.130$ $\pm$ $0.001$ & $0.010$ $\pm$ $0.001$ & $0.062$ $\pm$ $0.004$ & $0.010$ $\pm$ $0.001$ & $0.081$ $\pm$ $0.001$ & $0.008$ $\pm$ $0.000$ & $0.025$ $\pm$ $0.001$ & $0.004$ $\pm$ $0.000$ & $0.065$ $\pm$ $0.001$ \\ 
sat & $0.052$ & $0.012$ & $0.013$ & $0.004$ & $0.020$ & $0.041$ & $0.003$ & $0.016$ & $0.025$ \\ 
 & $0.053$ $\pm$ $0.001$ & $0.023$ $\pm$ $0.001$ & $0.033$ $\pm$ $0.002$ & $0.015$ $\pm$ $0.001$ & $0.024$ $\pm$ $0.001$ & $0.043$ $\pm$ $0.001$ & $0.023$ $\pm$ $0.002$ & $0.020$ $\pm$ $0.001$ & $0.027$ $\pm$ $0.001$ \\ 
vehicle & $0.056$ & $0.000$ & $0.031$ & $0.017$ & $0.043$ & $0.112$ & $0.041$ & $0.014$ & $0.018$ \\ 
 & $0.069$ $\pm$ $0.003$ & $0.119$ $\pm$ $0.009$ & $0.136$ $\pm$ $0.011$ & $0.058$ $\pm$ $0.003$ & $0.056$ $\pm$ $0.003$ & $0.126$ $\pm$ $0.005$ & $0.079$ $\pm$ $0.005$ & $0.047$ $\pm$ $0.003$ & $0.043$ $\pm$ $0.003$ \\ 
yeast & $0.110$ & $0.037$ & $0.053$ & $0.060$ & $0.104$ & $0.083$ & $0.024$ & $0.039$ & $0.071$ \\ 
 & $0.120$ $\pm$ $0.005$ & $0.152$ $\pm$ $0.009$ & $0.178$ $\pm$ $0.012$ & $0.092$ $\pm$ $0.006$ & $0.114$ $\pm$ $0.005$ & $0.114$ $\pm$ $0.006$ & $0.095$ $\pm$ $0.007$ & $0.075$ $\pm$ $0.006$ & $0.085$ $\pm$ $0.005$ \\ 
\bottomrule
\end{longtable}
\end{sidewaystable}

\begin{sidewaystable}
\footnotesize
\begin{longtable}{lrrrrrrrrr}
\caption{\label{tab:cts_results}
\large Continuous evaluation\\ 
\small First row for each dataset is absolute bias, second is RMSE\\ 
} \\ 
\toprule
& & \multicolumn{4}{c}{IPS} & \multicolumn{4}{c}{\bope} \\ 
 \cmidrule(lr){3-6}\cmidrule(lr){7-10}
Dataset & Direct Method & IPS & Doubly Robust & SWITCH & SWITCH-DR & \bope & Doubly Robust & SWITCH & SWITCH-DR \\ 
\midrule
abalone & $0.406$ & $0.116$ & $1.107$ & $0.090$ & $0.389$ & $0.114$ & $0.275$ & $0.151$ & $0.312$ \\ 
 & $0.408$ $\pm$ $0.003$ & $1.070$ $\pm$ $0.134$ & $2.879$ $\pm$ $0.261$ & $0.298$ $\pm$ $0.024$ & $0.393$ $\pm$ $0.005$ & $0.130$ $\pm$ $0.003$ & $0.287$ $\pm$ $0.005$ & $0.160$ $\pm$ $0.005$ & $0.317$ $\pm$ $0.006$ \\ 
admissions & $0.345$ & $0.029$ & $0.318$ & $0.264$ & $0.343$ & $0.045$ & $0.029$ & $0.023$ & $0.151$ \\ 
 & $0.358$ $\pm$ $0.006$ & $0.394$ $\pm$ $0.032$ & $1.137$ $\pm$ $0.105$ & $0.304$ $\pm$ $0.014$ & $0.356$ $\pm$ $0.010$ & $0.170$ $\pm$ $0.011$ & $0.257$ $\pm$ $0.015$ & $0.103$ $\pm$ $0.008$ & $0.223$ $\pm$ $0.013$ \\ 
airfoil & $0.690$ & $0.030$ & $0.723$ & $0.000$ & $0.479$ & $0.009$ & $0.016$ & $0.006$ & $0.195$ \\ 
 & $0.693$ $\pm$ $0.003$ & $0.128$ $\pm$ $0.022$ & $1.517$ $\pm$ $0.124$ & $0.045$ $\pm$ $0.004$ & $0.522$ $\pm$ $0.016$ & $0.030$ $\pm$ $0.002$ & $0.266$ $\pm$ $0.019$ & $0.027$ $\pm$ $0.003$ & $0.236$ $\pm$ $0.011$ \\ 
auto & $0.396$ & $0.028$ & $0.355$ & $0.315$ & $0.399$ & $0.013$ & $0.012$ & $0.030$ & $0.148$ \\ 
 & $0.406$ $\pm$ $0.005$ & $0.324$ $\pm$ $0.040$ & $1.272$ $\pm$ $0.138$ & $0.336$ $\pm$ $0.011$ & $0.409$ $\pm$ $0.010$ & $0.071$ $\pm$ $0.005$ & $0.181$ $\pm$ $0.010$ & $0.058$ $\pm$ $0.004$ & $0.199$ $\pm$ $0.015$ \\ 
housing & $0.538$ & $0.091$ & $0.411$ & $0.057$ & $0.537$ & $0.111$ & $0.404$ & $0.116$ & $0.445$ \\ 
 & $0.542$ $\pm$ $0.004$ & $1.367$ $\pm$ $0.483$ & $3.668$ $\pm$ $0.590$ & $0.172$ $\pm$ $0.056$ & $0.541$ $\pm$ $0.007$ & $0.130$ $\pm$ $0.007$ & $0.445$ $\pm$ $0.007$ & $0.123$ $\pm$ $0.002$ & $0.456$ $\pm$ $0.008$ \\ 
power & $0.108$ & $0.012$ & $0.023$ & $0.003$ & $0.057$ & $0.010$ & $0.012$ & $0.003$ & $0.022$ \\ 
 & $0.109$ $\pm$ $0.001$ & $0.060$ $\pm$ $0.020$ & $0.138$ $\pm$ $0.012$ & $0.015$ $\pm$ $0.003$ & $0.062$ $\pm$ $0.002$ & $0.040$ $\pm$ $0.007$ & $0.066$ $\pm$ $0.005$ & $0.013$ $\pm$ $0.002$ & $0.036$ $\pm$ $0.003$ \\ 
wine & $0.483$ & $0.253$ & $1.502$ & $0.282$ & $0.473$ & $0.197$ & $0.165$ & $0.155$ & $0.433$ \\ 
 & $0.488$ $\pm$ $0.004$ & $0.825$ $\pm$ $0.097$ & $3.737$ $\pm$ $0.524$ & $0.487$ $\pm$ $0.016$ & $0.478$ $\pm$ $0.007$ & $0.205$ $\pm$ $0.003$ & $0.191$ $\pm$ $0.005$ & $0.188$ $\pm$ $0.008$ & $0.439$ $\pm$ $0.007$ \\ 
\bottomrule
\end{longtable}
\end{sidewaystable}



\subsection{Data sources}

The sources for the datasets used in the experiments, along with necessary citations, can be found below.

\begin{sidewaystable}
\begin{tabular}{|l|l|}
\hline
Dataset                               & URL                                                                                       \\ \hline
ecoli                                 & \url{https://archive.ics.uci.edu/ml/datasets/ecoli}                                       \\
glass                                 & \url{https://archive.ics.uci.edu/ml/datasets/glass+identification}                        \\
letters                               & \url{https://archive.ics.uci.edu/ml/datasets/letter+recognition}                          \\
optdigits                             & \url{https://archive.ics.uci.edu/ml/datasets/optical+recognition+of+handwritten+digits}   \\
page-blocks                           & \url{https://archive.ics.uci.edu/ml/datasets/Page+Blocks+Classification}                  \\
pendigits                             & \url{https://archive.ics.uci.edu/ml/datasets/Pen-Based+Recognition+of+Handwritten+Digits} \\
satimage                              & \url{https://archive.ics.uci.edu/ml/datasets/Statlog+(Landsat+Satellite)}                 \\
vehicle\footnote{~\citep{data:vehicle}}           & \url{https://archive.ics.uci.edu/ml/datasets/Statlog+(Vehicle+Silhouettes)}               \\
yeast                                 & \url{https://archive.ics.uci.edu/ml/datasets/Yeast}                                       \\ \hline
abalone                               & \url{https://archive.ics.uci.edu/ml/datasets/abalone}                                     \\
admissions\footnote{\citep{data:admissions}}     & \url{https://www.kaggle.com/mohansacharya/graduate-admissions}                            \\
airfoil                               & \url{https://archive.ics.uci.edu/ml/datasets/airfoil+self-noise}                          \\
auto                                  & \url{https://archive.ics.uci.edu/ml/datasets/auto+mpg}                                    \\
housing                               & \url{https://www.kaggle.com/harlfoxem/housesalesprediction}                               \\
power\footnote{\citep{data:power1, data:power2}} & \url{https://archive.ics.uci.edu/ml/datasets/combined+cycle+power+plant}                  \\
wine\footnote{\citep{data:wine}}                 & \url{https://archive.ics.uci.edu/ml/datasets/wine+quality}                                \\ \hline
\end{tabular}
\end{sidewaystable}

\end{document}